\documentclass[a4paper]{article}
\pdfoutput=1
\usepackage{dblfloatfix}
\usepackage{caption}  
\usepackage{graphicx}
\usepackage[numbers]{natbib}
\usepackage{amssymb,amsmath,amsthm,amsfonts}

\usepackage{algorithm}
\usepackage{algorithmicx}
\usepackage{algpseudocode}

\usepackage{subfigure}
\usepackage{paralist}

\usepackage{listings}
\usepackage{url}

\newtheorem{theorem}{Theorem}
\newtheorem{corollary}{Corollary}
\newtheorem{definition}{Definition}

\newtheorem{example}{Example}

\algnewcommand{\algorithmicgoto}{\textbf{go to}}%
\algnewcommand{\GoTo}[1]{\algorithmicgoto~\ref{#1}}%

\newcommand{\keywords}[1]{\par\addvspace\baselineskip\noindent\textbf{Keywords: }\textit{#1}.}

\begin{document}
\title{Are there intelligent Turing machines?}

\author{N.~B\'atfai\thanks{
batfai.norbert@inf.unideb.hu, Department of Information Technology, University of Debrecen, H-4010 Debrecen PO Box 12, Hungary}}

\maketitle

\begin{abstract}
This paper introduces a new computing model based on the cooperation among Turing machines called orchestrated machines. Like universal Turing machines, orchestrated machines are also designed to simulate Turing machines but they can also modify the original operation of the included Turing machines to create a new layer of some kind of collective behavior. Using this new model we can define some interested notions related to cooperation ability of Turing machines such as the intelligence quotient or the emotional intelligence quotient for Turing machines.

\end{abstract}

\section{Introduction}

Supposedly we all can see roughly the same things if we look at the same things. 
This subjective observation naturally suggests that there are similar analytical capabilities and processes in our brain. 
From the point of view of cognitivism, we have similar cognitive architecture and the mental programs that are running in our brain are the same.
But we certainly have no direct experience of other persons' mental processes.
This was also emphasized by Eugene Wigner in \cite{wigner_essay} where it was pointed out that 
``\textit{our knowledge of the consciousness of other men is derived only through analogy and some innate knowledge}''.
But what is being sensed by human consciousness? It is the common-sense. 
The elements of it have been detected by consciousness. 
The existence of common-sense can be seen as the fruit of cooperation of human minds. 
At present, it is a popular research field to build 
databases similar to human common-sense 
(see, for example the projects Open Mind Common Sense and ConceptNet\cite{conceptnet} 
or WordNet\cite{wordnet} and Cyc\cite{cyc}) 
in order to computer programs also will be able to use common-sense knowledge.
But while collaboration between computers is a fully well-known area due to it is based on such protocols that were developed by engineers, cooperation between human minds is an open and interdisciplinary research field
and there is even the possibility that spoken and written communication are merely apparent 
where speech and reading/writing processes may mask the real channel that is based on biological quantum teleportation\cite{bqt}. A similar kind of situation may be observed in the case of communication between computers where several virtual communication layer protocols are built on the real physical link\cite{tanenbaum}. 

If strong AI\cite{strongAI} is true, then in principle, there might be algorithms that can compute all conscious and unconscious decisions of our life, at least theoretically.
For a long time we had been believing that it cannot be possible because a computer program cannot be rich enough to describe such kind of complex behavior. 
But this assumption is completely wrong because it is tacitly based on investigating the source code of computer programs that have been directly designed and  written by human programmers. 
To provide an intuitive counterexample of how complex the behavior of a computer program may be, consider the 5-state Busy Beaver\cite{Rado} champion Turing machine discovered by Marxen and Buntrock\cite{MarxenBuntrock}. 
This machine can execute many millions of steps before halting. 
The operation of it represents such level of complexity that cannot be achieved by using human programmers directly. That is, these and similar machines have not been created directly by human programmers, they have been only discovered by mankind\cite{bb_arxiv}.

Can the physical, chemical, and biological processes behind consciousness be considered to be similar, for example, to the Busy Beaver programs with complex behavior? 
They definitely can be, in the sense that these natural processes show complex behavior and the aim of natural sciences is precisely to uncover such processes. Unfortunately, we cannot provide any new recipes for discovering and uncovering these mysteries. But in this paper we will try to find sets of the simplest computer programs (aka Turing machines) whose members can perform complex and meaningful communication and cooperation among each other. We suppose that such machines exist and we are going to try to uncover them.

Neumann wrote in his last unfinished book \cite{Neumann} that ``\textit{It is only proper to realize that language is a largely historical accident.}'' The same might also be true for consciousness. In our view the minds are complex programs that can communicate and cooperate with each other where the cooperation is so important that consciousness cannot exist in itself. For this reason, developing some isolated standalone consciousness cannot be successful. Therefore, in the spirit of Chaitin's quote ``\textit{To me, you understand something only if you can program it. (You, not someone else!) Otherwise you don’t really understand it, you only think you understand it.}''\cite{Chaitin} we would like to develop ``networks'' of cooperating computer programs, but we are not going to try to create these from scratch because it seems almost impossible at this moment. Instead, we are going to find existing computer programs that can cooperate with each other by a given orchestrated model.  

\subsection{Notations and background}

We apply the definition of Turing machines that was also used in the paper \cite{running_time_arxiv}. 
According to the notation applied in this cited paper, let the quadruple $G=(Q_G, 0, \{0,1\}, f_G)$ be a Turing machine (TM) with the partial transition function 
$f_G:Q_G\times\{0,1\} \rightarrow Q_G\times\{0,1\}\times\{\leftarrow,\uparrow,\rightarrow\}$, $0 \in Q_G \subset \mathbb{N}$. 

Throughout of this paper, let $M $ denotes the set of Turing machines with given n or fewer states and let  
$F \to T \in f_G$ denotes a transition rule of the machine $G \in M$, where $F \in Q_G\times\{0,1\}$ and $T \in Q_G\times\{0,1\}\times\{\leftarrow,\uparrow,\rightarrow\}$.

Particular machines will also be given in the form of rule-index notation shown in \cite{running_time_arxiv}. For example, the famous 5-state champion Turing machine of Marxen and Buntrock can be given as  
(9, 0, 11, 1,  15, 2, 17, 3,  11, 4,  23, 5,  24, 6,  3,  7, 21, 9,  0)
where the first number is the number of rules and the other ones denote the ``from'' and ``'to' parts of the rules.
This notational form can be used directly in our C++ programs to create Turing machines 
as it is shown in the next code snippet 
\begin{lstlisting}
TuringMachine<5> mb1 (9, 0,11,1,15,2,17,3,11,4,23,5,24,
      6,3,7,21,9,0);
\end{lstlisting}
The programs and their several running logs can be found on a Github repository at \url{https://github.com/nbatfai/orchmach}.

\section{Orchestrated cooperation among Turing ma\-chines}

In intuitive sense, 
we call Turing machines that have the ability to cooperate with each other by using some kind of algorithms like Alg. \ref{orchmach1}, Alg. \ref{orchmach2} or Alg. \ref{orchmach3} \textit{orchestrated machines} (OMs). The idea behind these algorithms is to modify the original operation of Turing machines in order to evolve some collective behavior. The name \textit{orchestrated machines} is partly inspired by the Penrose -- Hameroff Orchestrated Objective Reduction (OrchOR) model of quantum consciousness \cite{orchor}.
The flavor of our algorithms in question is reminiscent of the dynamics of Neumann's U and R processes \cite{neumannq} and of the dynamics of the intuitive cellular automata example of \cite{orchor} in sense that one transition rule is chosen non-deterministically from applicable rules (R) then this selected rule will be executed in all Turing machines (U), and so on.

The first orchestrated algorithm (OM1) is shown in pseudocode in Alg. \ref{orchmach1}. It is intended to be applied in computer simulations of cooperation among Turing machines. Accordingly, this algorithm uses Turing machines that have no input. 
The second algorithm is given in pseudocode in Alg. \ref{orchmach2}, it may be used for study of standard questions such as, ``What is the language recognized by an orchestrated machine?'' Finally Alg. \ref{orchmach3} gives higher autonomy to individual Turing machines in their operation. It may be noted that all three algorithms can be considered as a special universal Turing machine.

\subsection{Orchestrated machines}

The complete pseudo code for orchestrated machines is shown in Alg. \ref{orchmach1}. 
As input, the algorithm gets a set of Turing machines.
The initial state of the orchestrated machine is that the heads of the contained Turing machines are reading zero and the machines are in their initial state which is 0.
The operation of the algorithm is controlled by the variable $F \in Q_G\times\{0,1\}$ initialized in Line \ref{Finit}.
In the main loop in Line \ref{mainloop}, every machines $G \in M_n$ determine the transition rule $F \to T \in f_{G}$ that can be applicable to actual value of $F$. 
If a machine $G$ has no such rule, then it halts and will be removed from the orchestrated machine in Line \ref{removeG}.
The main loop collects the right side of applicable rules of the machines $G \in M_n$ into the set $U$ in Line \ref{collectintoU} (the implementation uses a list instead of a set). 
After the inner loop is exited, only one right side will be non-deterministically chosen to be executed on all machines where it is possible as shown from Line \ref{selectexecfrom} to \ref{selectexecto}.
The orchestrated machine halts if all contained Turing machines halt.
The most precise description of Alg. \ref{orchmach1} can be found in its first implementation in class \texttt{OrchMach1} in orchmach1.hpp at \url{https://github.com/nbatfai/orchmach}.

\begin{algorithm}[h!]
\small
\caption{Orchestrated machines (with no input)\label{orchmach1}}
\begin{algorithmic}[1]
\Require $M_0 \subseteq M$, \Comment{$M_0$ is the investigated subset of the machines}
	\Statex $n \in \mathbb{N}$, $N\in \mathbb{N} \cup \{\infty\}$\Comment{Local variables}
	\Statex $M_n\subseteq M_{n-1}\subseteq M$, 
	\Statex $G_n \in M$, $T_n \in Q_{G_n}\times\{0,1\}\times\{\leftarrow,\uparrow,\rightarrow\}$.
\Ensure $N$, 
	\Statex $(card(M_2),\dots, card(M_{N-1}))$,  \Comment{$card(M_i)$ denotes the cardinality of the set $M_i$. }
	\Statex $((G_2, T_2), \dots, (G_{N-1}, T_{N-1}))$,  \Comment{If $N$ is finite, then the sequences $\{card(M_i)\}$ and $\{(G_i, T_i)\}$ are also finite and otherwise all three of them are infinite.}
\Procedure{$\operatorname{OrchMach1}$}{$M_0$}
\State $n = 0$ \Comment{Counter of steps}
\State $F=(0, 0)$ \Comment{The initial state is 0 and the input word is the empty word that is the current symbol under the head is 0}\label{Finit}
\While{$M_n \ne \emptyset$}\label{mainloop}	
\State $U=\emptyset$, $M_{n+1} = M_n$
\For{$G \in M_n$} \Comment{For all machines in $M_n$}
  \If{$F \to T \in f_{G}$} \Comment{$f_{G}$ contains only one or no such rule because $G$ is deterministic}
    \State $U = U \cup \{(G, T)\}$\label{collectintoU}
  \Else \Comment{G halts}
    \State $M_{n+1} = M_{n+1}  \setminus {G}$\label{removeG}
  \EndIf
\EndFor
    \State $(G_n, T_n) = select(U)$ \Comment{One rule is non-deterministically chosen}\label{selectexecfrom}
    \State $F = exec(G_n, T_n)$ \Comment{to be executed on the same machine from where the rule was selected}
    \For{$G \in M_{n+1} \setminus {G_n}$}
    	\State $exec(G, T_n)$ \Comment{to be executed on the other machines where it is possible (because it can happen that $G$ do not contain the state that is appeared on the right side of the rule $F \to T_n$)}\label{selectexecto}
    \EndFor
    \State $n = n + 1$
\EndWhile  \Comment{All machines have halted}
    \State $N = n$ \Comment{$card(M_0)=card(M_1)$ and if $N$ is finite, then $card(M_N)=0$}
    \State $o_2 = card(M_2), \dots, o_{N-1} = card(M_{N-1})$ 
    \State $O_2 = (G_2, T_2), \dots, O_{N-1} = (G_{N-1}, T_{N-1})$  \Comment{The selected transition rules (and their machines)}
    \State \Return{$N$} \Comment{or ``returns'' $\infty$ if the while loop begins in line 4 never ends.}
\EndProcedure
  \end{algorithmic}
\end{algorithm}

\begin{theorem}[The tapes are the same]\label{tapes}
Let $M_n \subseteq M$ be the set of Turing machines used in (line 4 of) Alg. \ref{orchmach1}.
The contents of the tapes of Turing machines $G \in M_n$ are the same.
\end{theorem}

\begin{proof}
\begin{inparaenum}
\item The statement holds trivially for the case $n = 0$ because tapes contain only zeros doe to the machines were started with no input.
\item  The tapes of the machines may be changed in line 14 and 16.
Assume that the statement holds for some $n$ then after the execution of $T_n$ the symbols under the heads of the machines will be the same.
\end{inparaenum}
\end{proof}

\begin{corollary}[One tape is enough]
It is important both theoretical but also from implementation aspects that computer simulation of orchestrated machines can be based on using only one common tape.
\end{corollary}

\begin{corollary}[$F$ would be computed locally]
In line 16, if  $G$ contains the state that is appeared on the right side of the rule $F \to T_n$ then $F$ (computed in line 14) would be equal to $exec(G, T_n)$. 
\end{corollary}

There are special cases in which the behavior of Alg. \ref{orchmach1} is equivalent with a non-deterministic Turing machine (NDTM). It is shown in the next theorem.

\begin{theorem}[Relation to the NDTMs]\label{relndtm}
Let $H \subseteq M$, $F_i = \{F \vert F \to T \in f_{G_i}\}$, $G_i \in H$, $i=1, \dots, k$.
If $F_i = F_j$, $1 \le i, j \le k$ then there is an equivalent NDTM to the orchestrated machine $\operatorname{OrchMach1}(H)$.
\end{theorem}

\begin{proof}
There are two possible way to run the machine $\operatorname{OrchMach1}$(H)
\begin{inparaenum}
\item there is $N \in \mathbb{N}$ that for all $n < N$, $M_n = M_{n-1}$, $M_n = \emptyset$
\item for all $n \in \mathbb{N}$ holds that $M_n = M_{n-1}$.
\end{inparaenum}
We can construct an equivalent NDTM as follows: $f_{NDTM} = \cup{F_i} = F_j, j \in \{1, \dots, k\}$.
\end{proof}

\begin{corollary}[Deterministic decomposition of NDTMs]\label{decomp}
For every NDTM there exists an equivalent OM1. 
\end{corollary}

\begin{proof}
The proof is constructive. Let $(a,x)$ be the left side of a non-deterministic rule of a NDTM with right sides 
$(a,x) = \{(b_1, y_1, d_1), \dots, (b_n, y_n, d_n)\}$ and let $T$ be a NDTM.
%Perform the following algorithm 
Perform the algorithm shown in Alg. \ref{proofalg}. 
%After this, for the breed $H$ satisfies the conditions of Theorem \ref{relndtm}.
After this, the breed $H$ satisfies the conditions of Theorem \ref{relndtm}.
\end{proof}

\begin{algorithm}[h!]
\small
\caption{The algorithm of the proof of Corollary \ref{decomp}\label{proofalg}}
\begin{algorithmic}[1]
\State $H = \{T\}$
\If{there exist a non-deterministic rule $(a,x)\to$ in machines in $H$} \label{istherendrule}
      \State A ruleset $(a,x)\to$ is selected and let $H' = \emptyset$.       
	\For{$G \in H$} \Comment{For all machines in $H$}
		\For{$i=1$ to $n$} \Comment{For all right sides of $(a,x)$}
			\State Construct a new DTM $T_i$ such that $f_{T_i} = f_G \setminus \{(a,x)\to\}$ and 
$f_{T_i} =f_{T_i} \cup \{(a,x) \to (b_i, y_i, d_i)\}$
                   \State $H' = H' \cup \{T_i\}$
		\EndFor	
	\EndFor
	\State $H = H'$
    \State \GoTo{istherendrule}
\EndIf
\end{algorithmic}
\end{algorithm}

A machine $\operatorname{OrchMach1}(H)$ \textbf{halts} if there is a computation of Alg. \ref{orchmach1} such that $\operatorname{OrchMach1}(H)$ $< \infty$.

\begin{definition}[Turing machine breeds]
The set $H \subseteq M$ is referred to as a Turing machine breed (or simply a breed) 
if  $\operatorname{OrchMach1}(H)$ halts, that is if there exist a finite sequence $o_n(H)$. A breed $H$ is called non-trivial if $\overline{o_n(H)} \ge 2$, where the overline denotes the mean value. 
\end{definition}

\begin{definition}[Convergence and divergence of machine breeds]
A machine breed $H$ is divergent if for all $K \in \mathbb{N}$ there exist $O_k(H)$ such that 
$\operatorname{OrchMach1}(H) \ge K$. A machine breed is convergent if it is not divergent.
\end{definition}

\begin{example}[A divergent breed]
Let 
$A=( \{0\}, 0, \{0,1\},  \{(0,0) \to (0, 1, \rightarrow)\}) \text{ and}$ and  
$B=( \{0\}, 0, \{0,1\},  \{(0,0) \to (0, 0, \leftarrow)\})$
be two Turing machines then it is easy to see that the breed $\{A, B\}$ is divergent. 
For example, let $k=K$ and the following transition rules have been chosen:
       $O_2=(A, (0,0) \to (0, 1, \rightarrow))$, $\dots$, $O_{k-1}=(A, (0,0) \to (0, 1, \rightarrow))$,
       $O_k=(B, (0,0) \to (0, 0, \leftarrow))$.
\end{example}

\begin{figure}[h]
\centering
\subfigure[An infinite loop $A$.]{\qquad\includegraphics[scale=.5]{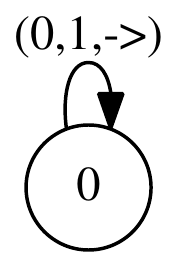}\qquad\label{figa}}
\qquad 
\subfigure[An other infinite loop $B$.]{\qquad\includegraphics[scale=.5]{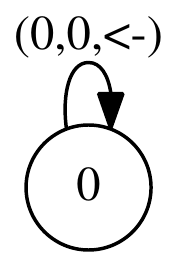}\qquad\label{figb}}
\qquad
\subfigure[The divergent breed consisting of two infinite loops $A$ and $B$.]{\qquad\includegraphics[scale=.5]{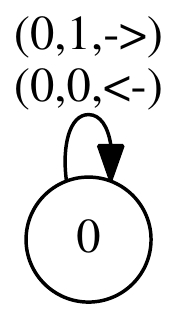}\qquad\label{figab}}
\caption{A divergent breed. The machines are given by their state transition diagrams.}
\label{figabab}
\end{figure}

\begin{example}[A convergent breed]\label{exapmle:cd}
Let 
$C=( \{0, 1, 2\}, 0, \{0,1\},  \{
(0,0) \to (1, 0, \rightarrow), 
(1,0) \to (2, 0, \rightarrow) 
\})$ and 
$D=( \{0, 1, 2\}, 0, \{0,1\},  \{
(0,0) \to (1, 1, \rightarrow), 
(1,0) \to (2, 1, \rightarrow) 
\})$ 
be two Turing machines shown in Fig. \ref{figcdcd}. It may be shown easily that the breed $\{C, D\}$ is convergent because it may be corresponded to the non-deterministic Turing machine shown in Fig. \ref{figcdcd} that always halts. 
\end{example}

\begin{figure}[h!]
\centering
\subfigure[The machine C.]{\qquad\includegraphics[scale=.5]{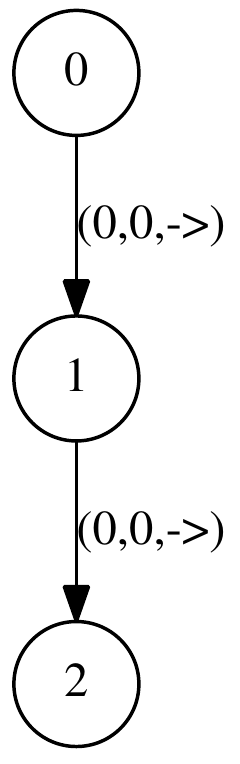}\label{figc}\qquad}
\qquad
\subfigure[The machine D.]{\qquad\includegraphics[scale=.5]{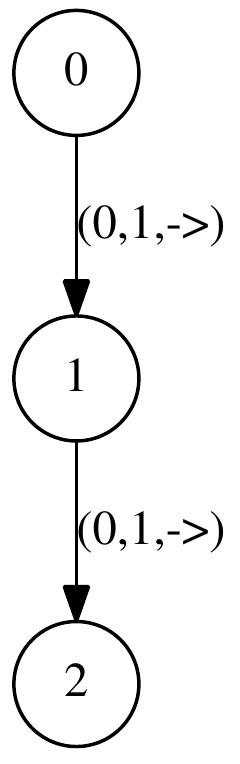}\label{figd}\qquad}
\qquad
\subfigure[The equivalent NDTM to the breed \{C, D\}.]{\qquad\includegraphics[scale=.5]{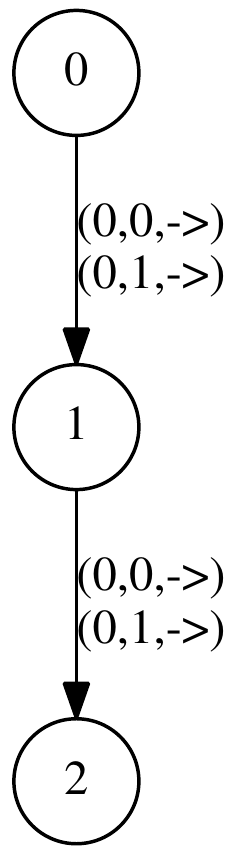}\label{figcd}\qquad}
\caption{A convergent breed.}
\label{figcdcd}
\end{figure}

\begin{example}[An infinite breed from finite machines]
The Turing machines
$E=( \{0\}, 0, \{0,1\},  \{(0,0) \to (0, 1, \uparrow)\})$ and 
$J=( \{0\}, 0, \{0,1\},  \{(0,1) \to (0, 0, \uparrow)\})$
shown in Fig. \ref{figefef} 
are not infinite loops but the breed $\{E, J\}$ is divergent. 
\end{example}

\begin{figure}[h!]
\centering
\subfigure[The machine $E$ is not an infinite loop.]{\qquad\includegraphics[scale=.5]{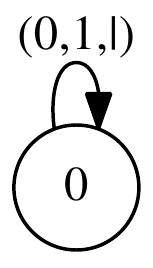}\label{fige}\qquad}
\qquad
\subfigure[The machine $J$ is also not an infinite loop.]{\qquad\includegraphics[scale=.5]{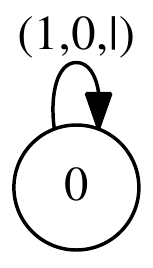}\label{figf}\qquad}
\qquad
\subfigure[The breed $\{E, J\}$ is infinite.]{\qquad\includegraphics[scale=.5]{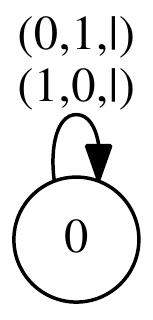}\label{figef}\qquad}
\caption{An infinite breed from finite machines.}
\label{figefef}
\end{figure}

\begin{theorem}[]
Let $H \subseteq M$ is a breed. If $H$ contains an infinite loop then it is divergent. 
\end{theorem}

\begin{proof}
The proof is trivial, simply select  the transition rule of the infinite loop machine in every step.
\end{proof}

\begin{theorem}[Halting of orchestrated machines]
The language 
of convergent breeds 
is algorithmically undecidable.
\end{theorem}

\begin{proof}
It is trivial because the sublanguage of trivial convergent breeds gives a well-known variant of the halting problem. (A breed $H$ is called trivial if $card(H) = 1$.)
\end{proof}

\begin{definition}[Purebred breeds]
A convergent machine breed $H$ is purebred
 if there is no real subset $M_1 \subset H$ such that  
$\{\operatorname{OrchMach1}(M_1,{'}T')\} = \{\operatorname{OrchMach1}(H, {'}T')\}$
where $'T'$ denotes a possible result of the computation $\operatorname{OrchMach1}$  
for precise details see from line 23 to line 31 in Alg. \ref{orchmach2ret}.
\end{definition}

\begin{algorithm}[H]
\small
\caption{Orchestrated machines (return values)\label{orchmach2ret}}
\begin{algorithmic}[1]
\Require $M_0 \subseteq M$, $m \in \{{'}N', {'}T', {'}O', {'}o'\}$\Comment{indicates what return value will be returned}
\Statex $\dots$
\Procedure{$\operatorname{OrchMach1}$}{$M_0$, m}
\Statex $\dots$
 \makeatletter\setcounter{ALG@line}{22}\makeatother
 \If{$m = {'}N'$} 
   \State return N
  \ElsIf {$m = {'}O'$}  
   \State return $O_2 = (G_2, T_2), \dots, O_{N-1} = (G_{N-1}, T_{N-1})$  
  \ElsIf {$m = {'}o'$}  
   \State return $o_2 = card(M_2), \dots, o_{N-1} = card(M_{N-1})$
  \ElsIf {$m = {'}T'$}  
   \State return the concatenation of tape symbols from the leftmost 1 to the rightmost 1 (or $\infty$ if N is equal to $\infty$)
 \EndIf
\EndProcedure
  \end{algorithmic}
\end{algorithm}

\begin{example}[A purebred breed and a not purebred one]
In Example \ref{exapmle:cd} the breed $\{C, D\}$ is purebred but if the machine $G=( \{0, 3\}, 0, \{0,1\},  \{(3,0) \to (3, 1, \rightarrow) \})$
is added to the breed $\{C, D\}$ then the result $\{C, D, G\}$ breed will be not purebred.
\end{example}

In the following, let $\mathcal{B}(M)$ denote the set of all purebred Turing machine breeds.

\begin{definition}[iq, eq]\label{iqeqdef}
Let $H \in \mathcal{B}(M)$ be a purebred breed, the quantity 
$\operatorname{iq}(H) = \max_{}\{\operatorname{OrchMach1}(H)\}$, $iq: \mathcal{B}(M) \to \mathbb{N}$
is called the \textit{intelligence quotient}  and similarly
the quantity 
$\operatorname{eq}(H) = \max_{}\{\left \lfloor \overline{o_n(H)} \right \rfloor\}$, $eq: \mathcal{B}(M) \to \mathbb{N}$
is called the \textit{emotional quotient} 
of the breed $H$.
\end{definition}

\begin{definition}[Intelligence functions]
Let $N, Z \in \mathbb{N}$ be na\-tu\-ral numbers.
The functions 
\begin{align*}
%eq: \mathbb{N} \to \mathbb{N}, \ &eq(N) = \\
%&\max_{H \in \mathcal{B}(M)}\{\left \lfloor \overline{o_n(H)} \right \rfloor \mid \operatorname{OrchMach1}(H) = N\}
&eq: \mathbb{N} \to \mathbb{N}, \ eq(N) = 
\max_{H \in \mathcal{B}(M)}\{\left \lfloor \overline{o_n(H)} \right \rfloor \mid \operatorname{OrchMach1}(H) = N\} 
\\
&iq: \mathbb{N} \to \mathbb{N}, \ iq(Z) = 
\max_{H \in \mathcal{B}(M)}\{\operatorname{OrchMach1}(H)\mid \left \lfloor \overline{o_n(H)} \right \rfloor  = Z\}
\end{align*}
are called \emph{intelligence functions} of breeds, 
where $\lfloor \rfloor$ denotes the floor function, 
but a more precise definition can be given that uses Def. \ref{iqeqdef} as follows

\begin{align*}
&EQ: \mathbb{N} \to \mathbb{N}, \ EQ(N) = 
\max_{H \in \mathcal{B}(M)}\{\left \lfloor \overline{o_n(H)} \right \rfloor \mid \operatorname{iq}(H) = N\} 
\\
&IQ: \mathbb{N} \to \mathbb{N}, \ IQ(Z) = 
\max_{H \in \mathcal{B}(M)}\{\operatorname{OrchMach1}(H)\mid \operatorname{eq}(H) = Z\}
\end{align*}
\end{definition}

In intuitive sense, the function $EQ(N)$ gives the maximum number of machines that can do $N$ steps together in a purebred breed, and inversely,  $IQ(Z)$ gives the maximum number of steps that can be done by $Z$ machines together as members of a purebred breed.  It is to be noted that functions $EQ$ and $IQ$ are well defined total functions due to $M$ is a finite set.  

\begin{theorem}[]
Let $x,y \in \mathbb{N}$ be na\-tu\-ral numbers, 
$\operatorname{EQ}(x) \ge y$ $\Leftrightarrow$ $\operatorname{IQ}(y) \ge x$.
\end{theorem}

\begin{proof}
It simply follows from the structure of the definitions of $\operatorname{EQ}$ and $\operatorname{IQ}$. 
\end{proof}

It is an open question whether or not there is an interesting relation between the functions $IQ$ and $EQ$.
At this point, we have just started to collect experiences with orchestrated machines.
We have developed a computer program to help automatically investigate Turing machine breeds.
The program can be found in the git repository at \url{https://github.com/nbatfai/orchmach}. It allows to gather experience with orchestrated machines. The results of some first experiments are presented in Table \ref{firstexperiment}, 
where the first column shows the cardinality of the examined breed. 
The second column is the maximum number of ones of a given breed's individual Turing machines. 
For example, the breed labelled by ``5a'' contains the Marxen and Buntrock's champion machine so its first ``1s'' column is $4097$. 
(We have used several well-known Busy Beaver TMs like Marxen and Buntrock's champion machines, Uhing's machines or Schult's machines \cite{MichelSurvey}. The exact lists of TMs of examined breeds and full running logs are available at \url{http://www.inf.unideb.hu/~nbatfai/orchmach} or in the sources at \url{https://github.com/nbatfai/orchmach}.)
The third column is the usual time complexity of the most powerful individual Turing machine contained in the previous column.
The other columns show running results that are organized in triplet groups of three, the first column of triplets is maximized for $o_2$, the second one is maximized for $OrchMach1(H, {'}N')$ and the last one is maximized for the number of ones.

For example, it is very interesting that with the exception of the first 21 time steps of the computation of the 75001 ones was done by Uhing's machine in 
$3.272948454 * 10^9$
time steps. (For further details, see related running log at the author's homepage  \url{http://www.inf.unideb.hu/~nbatfai/orchmach/data/breed_3a_sel.txt}.) 

It is likely a significant observation that there are breeds that are more powerful if their 
%$o_2 > 1$,
$o_2$ are greater than 1, for example, see the triplet $(6, 831, 59)$ in the row ``21'' of Table \ref{firstexperiment}. 
%Some similar situations are presented in Fig. \ref{rplots2}. 
This situation is well presented by the plot labelled exp. 7/breed 10 (``21'') in Fig. \ref{rplots2}. 

Finally, it should be noted that 
%the breeds from ``31'' to``9'' are divergent and 
the program suggests that it is possible that the listed breeds may be convergent.

\begin{table*}[h!]
\renewcommand{\arraystretch}{1.2}
\caption{The first computational results, as expected, suggest that a breed may be more powerful (in number of computed ones or in time complexity) than an individual Turing machine (see, for example 4096 $\to$ 75001, 32 $\to$ 9833 or 32 $\to$ 33161).
In addition, if $\overline{o_n}$ is increased to too large a value, then the powerful of the breeds is decreased (see, for example the triplet (3, 724, 118) $\to$ (12, 33, 19) in the row of ``13'').\label{firstexperiment}}
%\small
\footnotesize
\centering
\begin{tabular}{p{.6cm}|r|r||p{.5cm}|r|r||r|r|r||r|r|p{.6cm}}
&1s&$c_t$&max $\left \lfloor \overline{o_n} \right \rfloor $& N & 1s & $\left \lfloor \overline{o_n} \right \rfloor $ & max N & 1s & $\left \lfloor \overline{o_n} \right \rfloor $ & N & max 1s
\tabularnewline
\hline 
``7''&4097&7E7&
6& 31& 7&
1 &8.5E8 &14276&
1 &8.5E8 &14276
 \tabularnewline
\hline 
``3''&4097&7E7&
2& 33& 7&
1& 3.2E9& 74280&
1& 3.2E9& 74281
\tabularnewline
\hline 
``3a''&4096&2.3E7&
2& 41& 7&
1& 3.2E9& 75001&
1& 3.2E9& 75001
\tabularnewline
\hline 
``18''&4097&7E7&
16& 31& 7&
5& 1121& 65&
5& 856 &100
\tabularnewline
\hline 
``17''&4097&7E7&
15& 30& 4&
5& 1205& 68&
5& 686& 110
\tabularnewline
\hline 
``6''&4097&7E7&
5& 136& 11&
3& 1213 &151&
3& 1189& 196
\tabularnewline
\hline 
``5''&4096&2.3E7&
4& 135& 36&
3& 474& 91&
3& 390& 128
\tabularnewline
\hline 
``5a''&4097&1.1E7&
4& 80 &12&
1& 1.1E7 &4097&
1& 1.1E7& 4097
\tabularnewline
\hline 
``6a''&4096&1.1E7&
5&88& 22&
3& 694& 109&
3& 614 &141
\tabularnewline
\hline 
``17a''&501&1.3E5&
15& 33& 9&
5& 1175& 73&
5& 745& 107
\tabularnewline
\hline 
``13''&501&1.3E5&
12& 33& 19&
3& 1267& 35&
3& 724& 118
\tabularnewline
\hline 
``5b''&32&582&
4& 47&6&
1& 5.1E7& 9833&
1& 5.1E7 &9833
\tabularnewline
\hline 
``15''&32&582&
13& 21& 7&
1& 5.8E8& 33161&
1& 5.8E8& 33161
\tabularnewline
\hline 
``21''&32&582&
19& 32& 12&
6& 831& 59&
2& 273& 101
\tabularnewline
\hline 
``18a''&160&2E4&
17& 33& 8&
6& 940& 14&
2& 267& 99
\tabularnewline
\hline 
``17b''&160&2E4&
16& 30& 5&
6& 962& 26&
2& 229& 95
\tabularnewline

\end{tabular}
\end{table*}

\begin{figure}[!h]
\centering
\includegraphics[scale=.4]{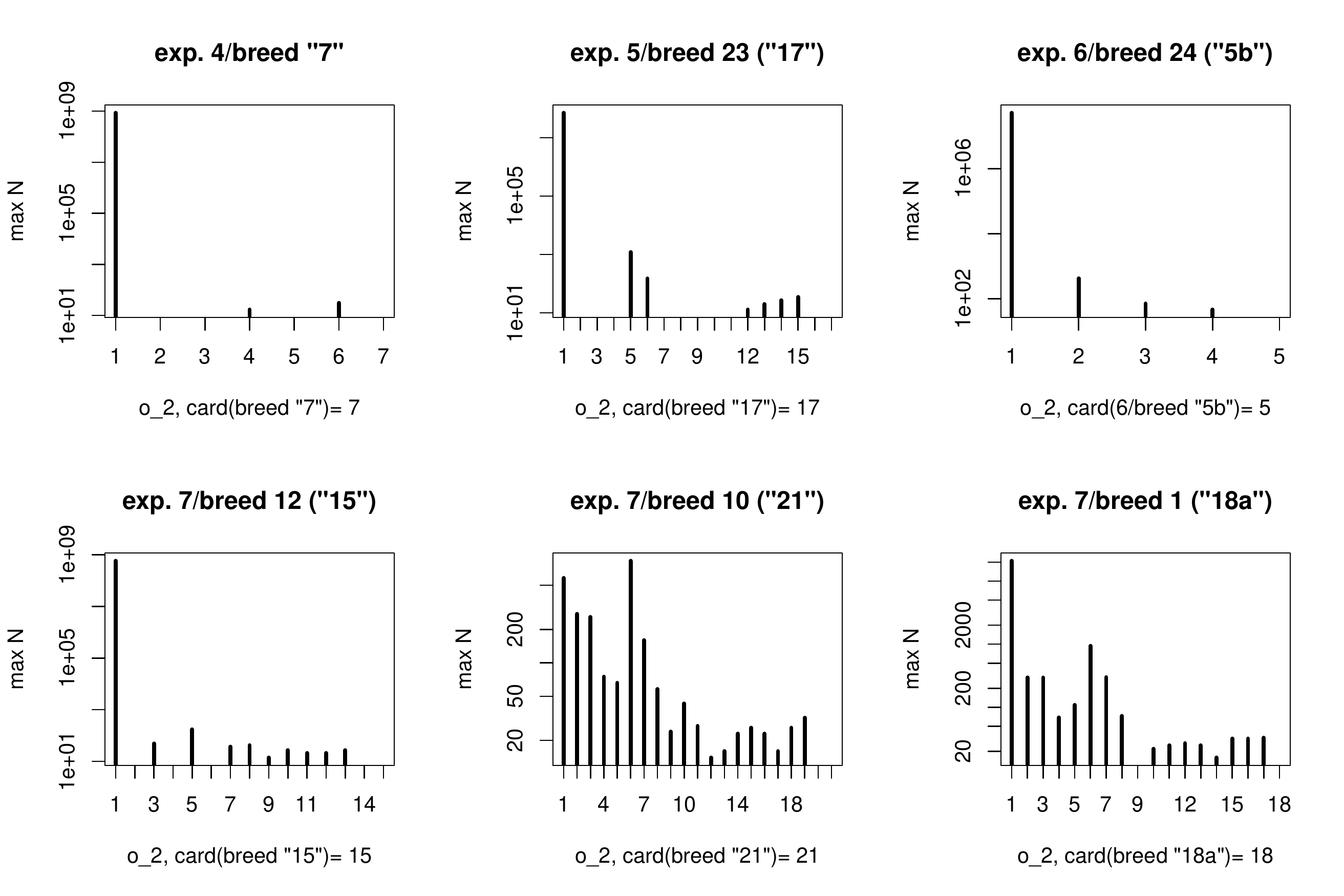}
\caption{The maximum computation length of some breeds were observed in our experiments by our search program based on random guessing. (Some ``trivial breed`` values, such as 70740809/breed ``17'', are not listed in Table \ref{firstexperiment} because they has not yet been found by the search program.)\label{rplots2}}
\end{figure}

Fig. \ref{rplots} shows further computational results of some probably convergent breeds. 
It is clear that $IQ(1) \ge 70740809$ because this estimation is based on the machine ( 9, 0, 11, 1,  15, 2, 17, 3, 1, 4, 23, 5, 24, 6, 3, 7, 21, 9, 0 ) as a trivial breed. (In practice, this machine is a variant of the Marxen and Buntrock's champion machine, see also \cite{bb_arxiv}). But our experiments have already found a probably convergent ``quasi-trivial'' breed, shown in Table \ref{firstexperiment}, that can produce $75001$ ones so it follows from this that $IQ(1) \ge 3.272948454*10^9$.  (A computation of a non-trivial breed is called quasi-trivial if its $o_2$ value is equal to $1$.)

\begin{figure}[!h]
\centering
\includegraphics[scale=.4]{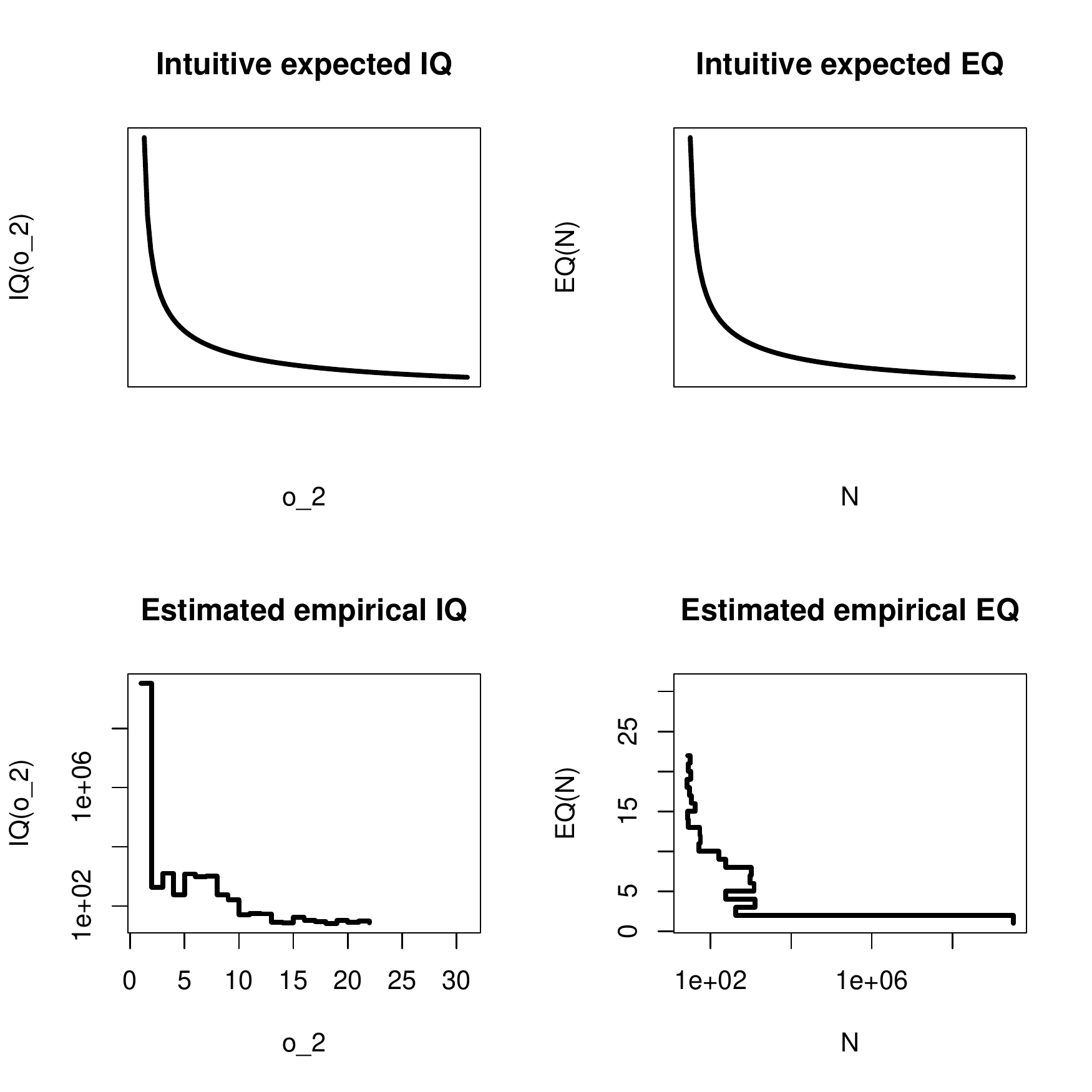}
\caption{The bottom two graphs plot the estimated empirical IQ functions for purebred breeds of Turing machines with 5 states. Some maximum values of $N$ are collected for $o_2=1, \dots, 22$ here. The top ones show the intuitive expected behavior of IQ functions of breeds.\label{rplots}}
\end{figure}

\begin{definition}[Intelligence quotients of Turing machines]\label{iq}
Let $V \in M$ be a Turing machine, the quantity 
$iq(V) = \max_{H \in \mathcal{B}(M)}\{\operatorname{OrchMach1}(H)\mid V \in H\}$
is called the \textit{intelligence quotient} and similarly
$eq(V) = \max_{H \in \mathcal{B}(M)}\{\overline{o_n(H)} \mid V \in H\}$
is called the \textit{emotional intelligence quotient} of the Turing machine $V$.
\end{definition}

\subsubsection{Orchestrated machines with input}

The previous algorithm used Turing machines that have a binary tape alphabet. To simplify constructing concrete Turing machines the $\{0, 1, \epsilon\}$ tape alphabet will be used in Alg. \ref{orchmach2} but the input alphabet will remain binary. The modification of Alg. \ref{orchmach1} is shown  in the pseudo code of Alg. \ref{orchmach2} where modified lines are contained only.

\begin{algorithm}[H]
\caption{Orchestrated machines (with the same input)\label{orchmach2}}
\begin{algorithmic}[1]
\Require $M_0 \subseteq M$, $w \in \{0, 1\}^*$\Comment{$w$ is the input word}
\Statex $\dots$
\Procedure{$\operatorname{OrchMach2}$}{$M_0$, $w$}
\Statex $\dots$
\makeatletter\setcounter{ALG@line}{2}\makeatother
\State $F=(0, a)$ \Comment{$a \in  \{0, 1\}$ is the first letter of the input word}
\Statex $\dots$
 \makeatletter\setcounter{ALG@line}{23}\makeatother
\EndProcedure
  \end{algorithmic}
\end{algorithm}

In the following let $w \in \{0, 1\}^*$ be a given arbitrary word. 
The $\operatorname{OrchMach1}(H)$ algorithm is a special case of the $\operatorname{OrchMach2}(H, w)$ where the input word $w$ is the empty word $\lambda$. With this in mind we can easily generalize the definitions of breeds as follows. 
A machine $\operatorname{OrchMach2}(H, w)$ \textbf{halts} if there is a computation of Alg. \ref{orchmach2} such that $\operatorname{OrchMach2}(H, w)$ $< \infty$.

\begin{definition}[w-Breeds]
The set $H \subseteq M$ is referred to as a Turing machine w-breed (or simply a w-breed) 
if  $\operatorname{OrchMach1}(H, w)$ halts, that is if there exist a finite sequence $o_n(H,w)$. A breed $H$ is called non-trivial if $\overline{o_n(H,w)} \ge 2$.
\end{definition}

\begin{definition}[Convergence and divergence]
A machine w-breed $H$ is divergent if for all $K \in \mathbb{N}$ there exist $O_k(H,w)$ such that 
$\operatorname{OrchMach1}(H,w) \ge K$. A machine breed is convergent if it is not divergent.
\end{definition}

\begin{definition}[The recognized language]
\begin{align*}
L(\operatorname{OrchMach2}(H)) = 
\{w \in \{0, 1\}^*  \mid  \operatorname{OrchMach2}(H, w)  \text{ halts}\}.
\end{align*}
\end{definition}

\begin{definition}[Purebred w-breeds]
A convergent machine w-breed $H$ is purebred
 if there is no real subset $M_1\subset H$ such that  
\[
L(\operatorname{OrchMach2}(H)) = L(\operatorname{OrchMach2}(M_1)).
\]
\end{definition}

\begin{example}[A purebred w-breed and a not purebred one]
Let 
$f_X=\{$
$(0,0) \to (1, 0, \rightarrow)$, 
$(0,1) \to (\infty, 1, \uparrow)$, 
$(1,1) \to (0, 1, \rightarrow)$, 
$(1,0) \to (\infty, 1, \uparrow)$, 
$(1,\epsilon) \to (\infty, 1, \uparrow)$, 
$(\infty,1) \to (\infty, 1, \uparrow)\}$ 
and 
$f_Y=\{$
$(0,1) \to (1, 1, \rightarrow)$, 
$(0,0) \to (\infty, 1, \uparrow)$, 
$(1,0) \to (0, 0, \rightarrow)$, 
$(1,1) \to (\infty, 1, \uparrow)$, 
$(1,\epsilon) \to (\infty, 1, \uparrow)$,
$(\infty,1) \to (\infty, 1, \uparrow)\}$
be transition rules of two Turing machines $X$ and $Y$ shown in Fig. \ref{figxy} then it is easy to see that for example the 0110-breed $\{X, Y\}$ is purebred. 
Let 
$f_{X'}=\{$
$(0,\epsilon) \to (r, \epsilon, \leftarrow)$, 
$(r,0) \to (r, 0, \leftarrow)$, 
$(r,1) \to (r, 1, \leftarrow)\}$ 
be transition rules of the machine $X'$. Then every $(..)^*\text{-breeds}$ $\{X, X', Y\}$ given by the regular expression $(..)^*$ are not purebred convergent $(..)^*\text{-breeds}$.
\end{example}

\begin{figure}[h!]
\centering
\subfigure[The language recognized by this machine $X$ is $(01)^n$.]{\qquad\includegraphics[scale=.4]{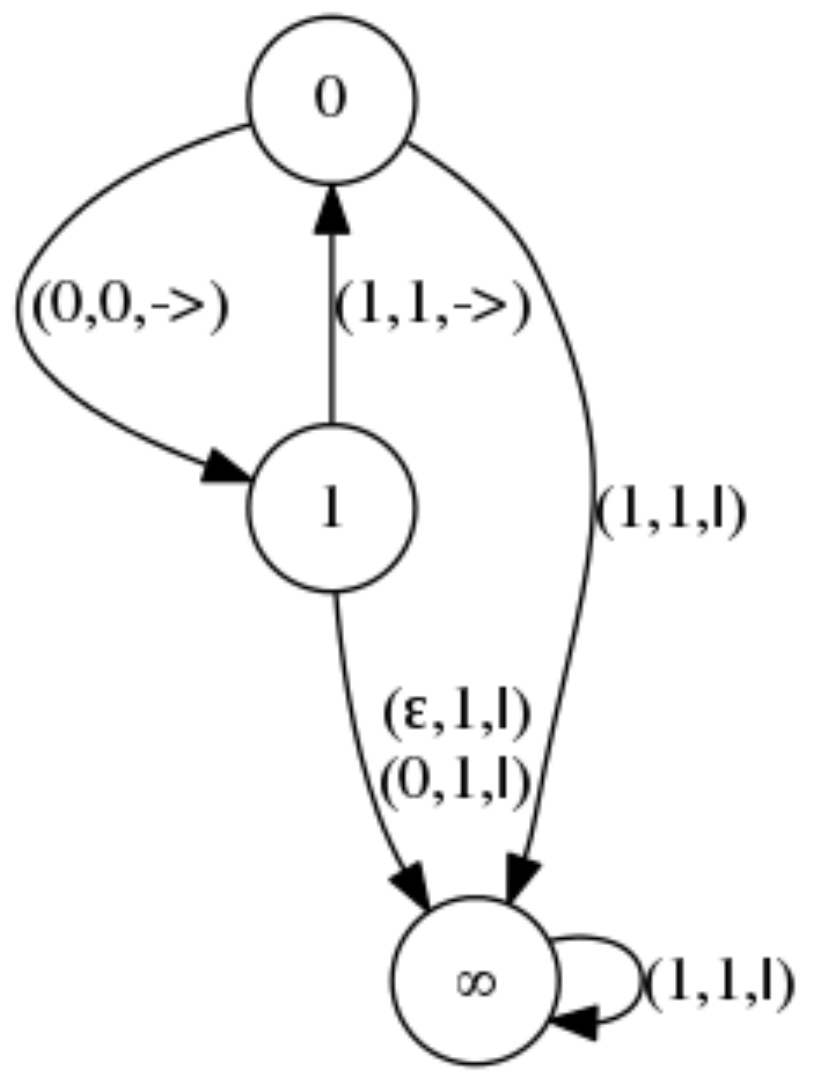}\label{figx}\qquad\qquad}
\qquad
\subfigure[The language recognized by $Y$ is $(10)^n$.]{\qquad\includegraphics[scale=.4]{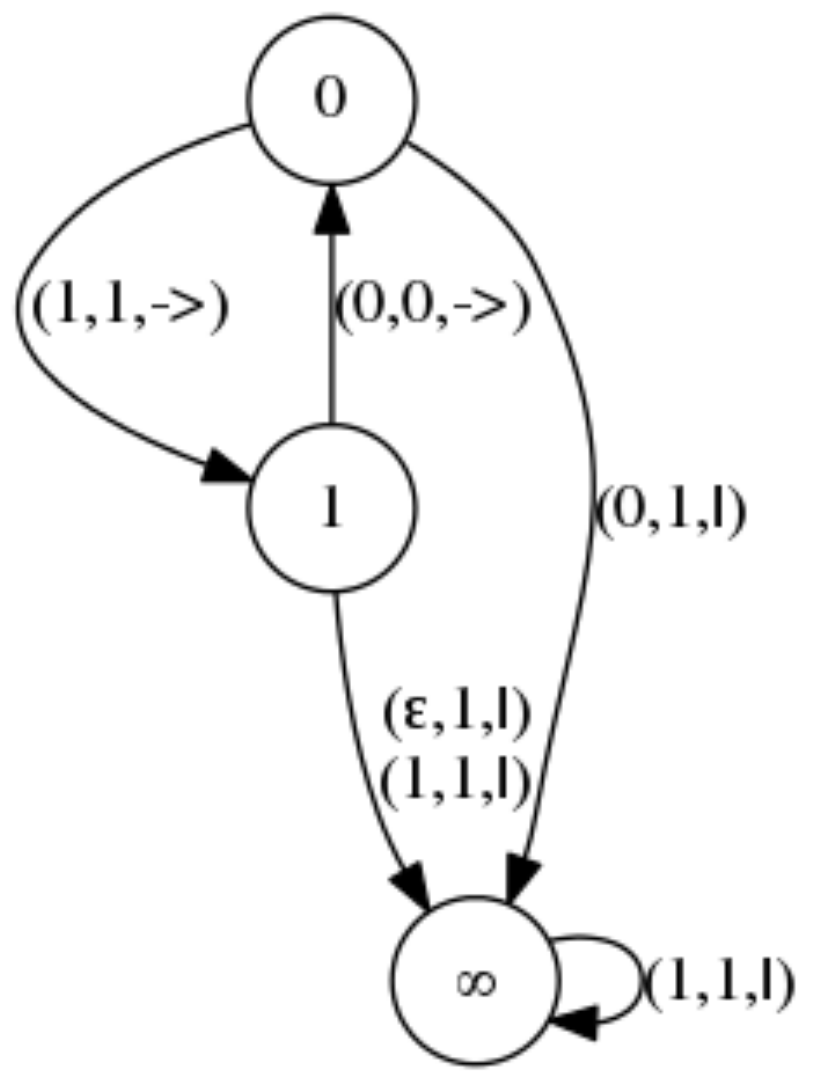}\label{figy}\qquad}
\caption{The Turing machines $X$ and $Y$.}
\label{figxy}
\end{figure}

\begin{example}[0110]
$0110 \notin L(X) \cup L(Y)$, but $0110 \in L(\operatorname{OrchMach2}(\{X, Y\}))$  
\end{example}

\begin{theorem}[Languages recognized by orchestrated machines]
$\bigcup{L(H_i)} \subseteq L(\operatorname{OrchMach2}(H))$
\end{theorem}

\begin{proof}
The proof is trivial. Suppose that $w \in L(H_j)$ and then simply select the transition rule of the machine $H_j$ in every step.
\end{proof}

In the following, let $\mathcal{B}(M,w)$ denote the set of all purebred Turing machine w-breeds.
But the former IQ definitions must be weakened or altered because the formal generalization of the formula in Def. \ref{iq}

$\operatorname{w-iq}(V) = \max_{H \in \mathcal{B}(M,w)}$$\{\operatorname{OrchMach2}(H, w)\mid V \in H\}$
would naturally lead to
$\operatorname{iq}(V) = \lim_{\vert z \vert \to \infty}{\frac{\operatorname{z-iq}(V)}{f(\vert z \vert)}}, 
H \in \mathcal{B}(M,z), V \in H$
but this would be infinite for all Turing machines because for every breed $H$ there exist a breed $H'$ such that $L(H)=L(H')$ and $\operatorname{OrchMach2}(H, w) \ge \vert z \vert f(\vert z \vert)$.

\begin{definition}[w-iq]
Let $H \in \mathcal{B}(M,w)$ be a Turing machine purebred breed, the quantity 
$\operatorname{w-iq}(H) = \max_{}\{\operatorname{OrchMach2}(H, w)\}$, $\operatorname{w-iq}: \mathcal{B}(M,w) \to \mathbb{N}$ is called the \textit{w-intelligence quotient} and similarly
the quantity $\operatorname{w-eq}(H) = \max_{}\{\left \lfloor \overline{o_n(H)} \right \rfloor\}$,  $\operatorname{w-eq}: \mathcal{B}(M,w) \to \mathbb{N}$
is called the \textit{w-emotional quotient} 
of the breed $H$.
\end{definition}

\begin{example}[]\label{iqx}
$\operatorname{w-iq}(\{X\}) \ge \vert w \vert$, 
$\operatorname{w-iq}(\{X,Y\}) \ge \vert w \vert$, 
$\operatorname{w-iq}(\{X'\}) \ge 2\vert w \vert$
 \end{example}

\begin{definition}[w-intelligence functions]
Let $N, Z \in \mathbb{N}$ be na\-tu\-ral numbers.
The functions 
\begin{align*}
&\operatorname{w-EQ}: \mathbb{N} \to \mathbb{N},\  \operatorname{w-EQ}(N) = 
\max_{H \in \mathcal{B}(M,w)}\{\left \lfloor\overline{o_n(H,w)}\right \rfloor \mid \operatorname{w-iq}(H) = N\}
\\
&\operatorname{w-IQ}: \mathbb{N} \to \mathbb{N},\  \operatorname{w-IQ}(Z) = 
\max_{H \in \mathcal{B}(M,w)}\{\operatorname{OrchMach2}(H,w)\mid  \operatorname{w-eq}(H) = Z\}
\end{align*}
are called \emph{w-intelligence functions} of breeds.
\end{definition}

\subsection{Universal orchestrated machines}

Alg. \ref{orchmach3} gives the pseudo code for \textit{universal orchestrated machines}. It allows higher autonomy to individual Turing machines in their operation. If we compare this orchestrated algorithm with algorithms given in previous sections we will easily see the difference, the variable $F$ that represents the actual state and the read symbol is a local variable in sense that each machines of a breed have their own variable $F$.  In this paper we do not investigate the properties of $\operatorname{OrchMach3}$ only the algorithm is presented in Alg. \ref{orchmach3}, but it is clear that Theorem \ref{tapes} does not hold for universal orchestrated machines.

\begin{algorithm}[H]
\small
\caption{Orchestrated machines (with different input)\label{orchmach3}}
\begin{algorithmic}[1]
\Require $M_0 \subseteq M$, $w_i \in \{0, 1\}^*, i=1,\dots,card(M_0)$\Comment{$w_i$ is the input word of the machine $R_i \in M_0$}
	\Statex $n \in \mathbb{N}$, $N\in \mathbb{N} \cup \{\infty\}$, $M_n\subseteq M_{n-1}\subseteq M$, $G_n \in M$, $T_n \in Q_{G_n}\times\{0,1\}\times\{\leftarrow,\uparrow,\rightarrow\}$, $F_G \in Q_G \times \{0, 1\}$,  $G \in M_n$ \Comment{Local variables}.	
\Ensure $N$, 
	\Statex $(card(M_2),\dots, card(M_{N-1}))$, $((G_2, T_2), \dots, (G_{N-1}, T_{N-1}))$.
\Procedure{$\operatorname{OrchMach3}$}{$M_0$, $\{w_i\}$}
\State $n = 0$
\State $F_G=(0, a)$, $G \in M_n$\Comment{$a \in  \{0, 1\}$ is the first letter of the input of $G$}
\While{$M_n \ne \emptyset$}	
\State $U=\emptyset$, $M_{n+1} = M_n$
\For{$G \in M_n$}
  \If{$F_G \to T \in f_{G}$}
    \State $U = U \cup \{(G, T)\}$
  \Else
    \State $M_{n+1} = M_{n+1}  \setminus {G}$ 
  \EndIf
\EndFor
    \State $(G_n, T_n) = select(U)$
    \For{$G \in M_{n+1}$}
    	\State $F_G = exec(G, T_n)$ 
    \EndFor
    \State $n = n + 1$
\EndWhile  
    \State $N = n$, $o_2 = card(M_2), \dots, o_{N-1} = card(M_{N-1})$ , $O_2 = (G_2, T_2), \dots, O_{N-1} = (G_{N-1}, T_{N-1})$ 
    \State \Return{$N$}
\EndProcedure
  \end{algorithmic}
\end{algorithm}

\section{Conclusion and future directions}

We have introduced a new special type of universal Turing machine called orchestrated machine that allows to begin the investigation of an a'priori ability of certain Turing machines to work with each other. 
Among purebred machine breeds we have defined two non-computable total functions EQ and IQ to catalyze the search for more  interesting machine breeds.

In this paper, the time complexity classes of orchestrated machines were not being investigated  but it is clear that the $\text{NP} \subseteq \text{OM1P}$, where $\text{OM1P}$ denotes the class of decision problems solvable by an orchestrated machine (with algorithm $\operatorname{OrchMach1}$) in polynomial time. This inclusion follows from the Theorem \ref{decomp}.

We have many other exciting and unanswered questions for future research. For example, in this paper, the orchestrated machines have been built from Turing machines.  Is it possible for an orchestrated machine to be constructed from orchestrated machines?

We believe machine breeds would become a good model for processes that can be well distinguished from each other. As an intuitive example, in a living cell, several well-known processes (such as reverse transcription or  citric acid cycle) are taking place in the same time, these processes together can be seen as a ``breed''. 
To illustrate our intuition, we can provide a further subjective example of a``breed''. When the author listens to Ferenc Liszt's Hungarian Rhapsody, for piano No. 15 (R\'ak\'oczi marsch) then the author's ``sensor program'' (that listens to the music) and Liszt's ``generator program'' (that wrote the music) may form a ``breed''.
As a first step towards in this direction, it has already been an interesting and unanswered question whether the orchestrated architecture introduced in this paper can be developed for higher-level programming models.

The main question is that how can we develop computing architectures that will be able to replace sequential nature of the Neumann architecture in some a'priori AI application domain, for example, in reproducing the human thinking. We have given a new model for Turing machines. But it is also true that at this point we can imagine only evolutionary programming methods for general programming of orchestrated machines. 

\section*{Acknowledgment}

The computations shown in this paper were partially  performed on the NIIF High Performance Computing supercomputer at University of Debrecen.

\bibliography{orchmach123}

\end{document}